\documentclass[twocolumn,10pt]{article}

\usepackage{microtype}
\usepackage{graphicx}
\usepackage{subfigure}
\usepackage{booktabs} 
\usepackage{stmaryrd}
\usepackage{amsmath}
\usepackage{amssymb}
\usepackage{amsthm}
\usepackage{natbib}

\newtheorem{lemma}{Lemma}
\newtheorem{theorem}{Theorem}

\newtheorem{condition}{Condition}

\newcommand\numberthis{\addtocounter{equation}{1}\tag{\theequation}}

\usepackage{hyperref}

\title{Generalized Sliced Distances for Probability Distributions}
\date{}

\author{Soheil Kolouri$^1$\and
Kimia Nadjahi$^2$\and
Umut \c{S}im\c{s}ekli$^2$\and
Shahin Shahrampour$^3$\\\\
$^1$ HRL Laboratories, LLC, Malibu, CA, USA\\
$^2$ LTCI, T\'{e}l\'{e}com Paris, Institut Polytechnique de Paris, France \\
$^3$ Texas A\&M University, College Station, TX, USA}

\begin{document}
\maketitle

\def\R{\mathbb{R}}
\def\S{\mathbb{S}}
\def\E{\mathbb{E}}
\def\bP{\mathbb{P}}
\def\bQ{\mathbb{Q}}
\def\cF{\mathcal{F}}
\def\N{\mathcal{N}}
\def\cR{\mathcal{R}}
\def\cG{\mathcal{G}}
\def\cX{\mathcal{X}}
\def\cH{\mathcal{H}}
\def\cU{\mathcal{U}}

\begin{abstract}
Probability metrics have become an indispensable part of modern statistics and machine learning, and they play a quintessential role in various applications, including statistical hypothesis testing and generative modeling. However, in a practical setting, the convergence behavior of the algorithms built upon these distances have not been well established, except for a few specific cases. In this paper, we introduce a broad family of probability metrics, coined as Generalized Sliced Probability Metrics (GSPMs), that are deeply rooted in the generalized Radon transform. We first verify that GSPMs are metrics. Then, we identify a subset of GSPMs that are equivalent to maximum mean discrepancy (MMD) with novel positive definite kernels, which come with a unique geometric interpretation. Finally, by exploiting this connection, we consider GSPM-based gradient flows for generative modeling applications and show that under mild assumptions the gradient flow converges to the global optimum. We illustrate the utility of our approach on both real and synthetic problems. 
%
\end{abstract}

\section{Introduction}
\label{sec:introduction}

Measuring the discrepancy between probability distributions is at the heart of statistics and machine learning problems. A classic example in statistics is the hypothesis testing in higher dimensions, which has attracted a plethora of interest in recent years \cite{gretton2012kernel,ramdas2017wasserstein,chwialkowski2015fast}.  Similarly, in generative modeling,  leveraging probability metrics and discrepancy measures as an alternative to the adversarial networks, used in Generative Adversarial Networks (GANs), has become an exciting topic \cite{dziugaite2015training,mohamed2016learning,li2017mmd,arjovsky2017wasserstein}. Notably, variations of the Wasserstein distances and the Maximum Mean Discrepancy (MMD) have enjoyed ample attention from the community and have incited many enthralling works in the literature.

There are specific challenges with measuring the discrepancy between two high-dimensional probability distributions, including the high computational cost (e.g., for p-Wasserstein distances), and growing sample complexity, i.e., in the sense of the dependence of convergence rate of a given metric between a measure and its empirical counterpart on the number of samples \cite{genevay2019sample}. The community has tackled these challenges from different angles in recent years. One of the thought-provoking approaches is via slicing high-dimensional distributions over their one-dimensional marginals and comparing their marginal distributions \cite{kolouri2019sliced,nadjahi2019asymptotic}. The idea of slicing distributions is related to the Radon transform and has been successfully used in, for instance, sliced-Wasserstein distances in various applications \cite{rabin2011wasserstein,kolouri2016sliced,carriere2017sliced,deshpande2018generative,kolouri2018sliced,nadjahi2019approximate}. More recently, \citet{kolouri2019generalized} extended the idea of linear slices of distributions, used in sliced-Wasserstein distances, to non-linear slicing of high-dimensional distributions, which is rooted in the generalized Radon transform.

In this paper, we leverage the idea of slicing high-dimensional distributions and introduce a broad family of probability metrics named Generalized Sliced Probability Metrics (GSPMs). We provide a geometric interpretation of these metrics, and show their connection to the well-celebrated MMDs. GSPMs are built based on the idea of 'slicing' high-dimensional distributions, or the pushforward measure of the high-dimensional input distributions for a real function. We emphasize that GSPMs are \emph{not} a subclass of Integral Probability Measures (IPMs) \cite{muller1997integral,dziugaite2015training}, however, they share many commonalities.  We show that a subset of GSPMs is equivalent to MMDs, and leverage this connection to define geometrically interpretable kernels for MMDs that were not explored in the literature prior to this work. 

Finally, following the work of \citet{arbel2019maximum}, we identify some regularity conditions under which we show that the introduced kernels, which are rooted in GSPMs, satisfy the conditions for the global convergence of gradient flows. Hence, the proposed kernels are suitable for applications dealing with probability flows and implicit generative modeling.

\section{Preliminaries}
\label{sec:prelim}

Let $\mu$ and $\nu$ be probability measures defined on a measurable space, $\cX$, with corresponding densities $p$ and $q$. In addition, let $(\cX,d)$ denote a metric space. Let $\cF$ be a class of real-valued bounded measurable functions on $\cX$. Then the slice of a probability measure $\mu$, with respect to $f\in\cF$, is the pushforward measure $f_\#\mu$. We use the equivalent terminology that the slice of a $d$-dimensional probability density function $p$ ($d\geq2$), with respect to a function $f\in\cF$, is a one-dimensional probability density function that is defined as: 
\begin{eqnarray}
    p_f(\cdot)&=&\int_\cX \delta(\cdot-f)d\mu\nonumber\\
    &=&\int_\cX p(x)\delta(\cdot-f(x))dx
    \label{eq:gslice}
\end{eqnarray}
where $\delta$ is a one-dimensional Dirac function. Intuitively, $p_f$ is the distribution of $f(x)$ (which is scalar) when $x$s are i.i.d samples from $p$, $x\overset{i.i.d}{\sim} p$.

\subsection{Radon Transform}
In Radon transform, we are interested in the question of whether one can recover the distribution $p$ from its slices $\{p_f:\forall f\in \cF\}$. In other words, when does the set $\{p_f:\forall f\in \cF\}$ preserve the information contained in $p$?  

{\bf Classical Radon Transform:} Denote by $\S^{(d-1)}:=\{\theta:\|\theta\|_2=1\}$ the unit sphere in a $d$-dimensional Euclidean space. The classical Radon transform shows that when the function class is ``linear'', i.e.,  $\cF=\{f(x)=\langle x , \theta\rangle: \forall x\in\R^d, \forall\theta\in\S^{(d-1)}\}$, the corresponding slices, $\cR p:=\{p_{f_\theta} : \forall \theta\in\S^{(d-1)} \}$ contain all the required information to recover the distribution $p$. The previous statement implies that the classical Radon transform map is invertible, i.e. we have
\begin{eqnarray}
&{\small\text{forward~:~}}& p_{f_\theta}(t)=\int_\cX p(x)\delta(t-\langle x\cdot \theta\rangle)dx, \\
&{\small \text{Inverse~:~}}& p(x)=\int_{\S^{(d-1)}} (p_{f_\theta}*\eta)(\langle x,\theta\rangle)d\theta
\end{eqnarray}
for $\theta\in\S^{(d-1)}$ and $t\in\R$, where $\eta(\cdot)$ is a one-dimensional high-pass filter with a Fourier transform $\hat{\eta}(\omega) =  c|\omega|^{d-1}$, appearing as a result of the Fourier slice theorem.  The geometric interpretation of this process is that $p_{f_\theta}(t)$ integrates $p$ along the hyperplane $H=\{x: \langle x,\theta\rangle=t\}$.

{\bf Generalized Radon Transform:} Classical Radon transform can be extended to the generalized Radon transform (GRT) to integrate $p$ over hypersurfaces \textit{i.e.} $(d-1)$-dimensional manifolds, $H=\{x: \langle x,f_\theta(x)=t\}$. The literature on GRT focuses on parametric functions $f_\theta$ defined on $\cX \times \Omega_\theta$ with $\cX \subseteq \mathbb{R}^d$ and $\Omega_\theta\subseteq (\mathbb{R}^n \backslash \{ 0 \})$. These functions are so-called ``defining functions''. To ensure that GRT is invertible, the following necessary conditions are identified \cite{homan2017injectivity}:
\begin{enumerate}
    \item $f_\theta$ must be a real-valued $C^\infty$ function on $\cX \times \Omega_\theta$ to guarantee the smoothness of hyper-surfaces,
    \item $f_\theta$ must be homogeneous of degree one in $\theta$, \textit{i.e.},
        $\forall \lambda \in \R,\; f_{\lambda\theta} = \lambda f_\theta$. The condition is required to guarantee a unique parametrization of hypersurfaces,
    \item  $f_\theta$ must be non-degenerate in the sense that $\nabla_x f_\theta\neq 0$. The non-degenerate assumption ensures that the $(d-1)$-dimensional hypersurfaces  do not collapse to points, and the integrals are well defined, 
    \item The mixed Hessian of $f_\theta$ must be strictly positive, i.e., $det(\nabla_\theta\nabla_x f_\theta(x))>0$ for $\forall x \in \cX$, and $\forall \theta\in\Omega_\theta$. This condition is a local form of the Bolker's condition (See \cite{homan2017injectivity}), which allows one to locally identify $(x,\theta)$ with the covector $\frac{\nabla_x f_\theta(x)}{\|\nabla_x f_\theta(x)\|}$. 
\end{enumerate}
The linear function class $\cF:=\{f_\theta(x)=\langle x,\theta\rangle~:~\forall x \in \cX,~\forall \theta\in\S^{(d-1)} \}$ is one example of such family of ``defining functions''. Invertibility of GRTs is a long standing research problem. We provide below a number of well-studied classes of ``defining functions'', that ensure invertibility of GRTs.

In \cite{kuchment2006generalized}, it is shown that the circular defining function, $f_\theta(x) = \|x-s*\theta\|_2$ with $s\in\mathbb{R}^+$ and $\theta \in  \S^{d-1}$ provides an injective GRT. Homogeneous polynomials with an odd degree also define an injective GRT \cite{ehrenpreis2003universality}, \textit{i.e.}
    $f_\theta(x) = \sum_{|\alpha| = m} \theta_\alpha x^\alpha$,
where we use the multi-index notation $\alpha = (\alpha_1, \dots, \alpha_{d_\alpha}) \in \mathbb{N}^{d_\alpha}$, $|\alpha| = \sum_{i=1}^{d_\alpha} \alpha_i$, and $x^\alpha = \prod_{i=1}^{d_\alpha} x_i^{\alpha_i}$. 
The summation here iterates over all possible multi-indices $\alpha$, such that $|\alpha| = m$, where $m$ represents the  polynomial degree and $\theta_\alpha \in \mathbb{R}$. The parameter set for homogeneous polynomials is then set to be $\mathbb{S}^{d_\alpha-1}$. One can see that the choice of $m=1$ recovers the linear case $\langle x,\theta\rangle$, in that the set of the multi-indices with $|\alpha|=1$ becomes $\{ (\alpha_1, \dots, \alpha_d); \alpha_i = 1 \text{ for a single } i\in \llbracket 1, d \rrbracket, \text{ and } \alpha_j = 0, \quad \forall j \neq i\}$ and includes $d$ elements. We note that GRT was also the basis for the recently proposed generalized sliced-Wasserstein distances \cite{kolouri2019generalized}.

\section{Generalized Sliced Probability Metrics (GSPMs)}

In this section, we show that any probability metric between one-dimensional probability measures can be extended to higher-dimensions via the concept of generalized slicing. Let $\xi(\cdot,\cdot)$ be a metric for one-dimensional probability measures. Then, for probability measures $\mu$ and $\nu$ defined on $\cX\subset\R^d$ with respective densities $p$ and $q$, the proposed GSPM is defined as follows:
\begin{equation}
\zeta_\cF(p,q):=\left(\int_{\Omega_\theta} \xi^r(p_{f_\theta},q_{f_\theta}) d\theta\right)^{\frac{1}{r}}
\label{eq:gspm}
\end{equation}
where $r\geq 1$. Let us first show that GSPM is a metric. 
Non-negativity and symmetry immediately follow from non-negativity and symmetry of $\xi(\cdot,\cdot)$, while triangle inequality follows from the Minkowski inequality:
\begin{eqnarray*}
\zeta_\cF(p,q)&=&\left( \int_{\Omega_\theta} \xi^r(p_{f_\theta},q_{f_\theta}) d\theta\right)^{\frac{1}{r}} \\
&\leq&\left( \int_{\Omega_\theta} \left( \xi(p_{f_\theta},h_{f_\theta})+\xi(h_{f_\theta},q_{f_\theta})\right)^r d\theta\right)^{\frac{1}{r}} \\
&\leq& 
\left( \int_{\Omega_\theta}  \xi^r(p_{f_\theta},h_{f_\theta})d\theta\right)^{\frac{1}{r}}+ \\ && \left(\int_{\Omega_\theta} \xi^r(h_{f_\theta},q_{f_\theta})d\theta\right)^{\frac{1}{r}}\\
&=& \zeta_\cF(p,h)+\zeta_\cF(h,q)
\end{eqnarray*}
Finally, the identity of indiscernibles states that, $\zeta_\cF(p,q)=0$ if and only if (iff) $p=q$. The forward proof is straightforward: $p=q$ results in $p_{f_\theta}=q_{f_\theta}$ and since $\xi$ is a metric $\xi(p_{f_\theta},q_{f_\theta})=0$ for $\theta\in\Omega_\theta$. If $p_{f_\theta}=q_{f_\theta}$  for $\theta\in\Omega_\theta$, we can conclude that $p=q$ iff the GRT is injective. Hence, if GRT is injective then GSPMs provide a metric. Otherwise, GSPMs are pseudo-metrics.  

\subsection{Max-GSPM}
 Equation \eqref{eq:gspm} is based on the expected value of $\xi^r(p_{f_\theta},q_{f_\theta})$, when $\theta\sim \cU_{\Omega_\theta}$ where $\cU_{\Omega_\theta}$ is the uniform distribution on $\Omega_\theta$. Here we show that the max version of GSPMs are also metrics. Substituting the expected value with supremum, leads to a metric defined as:
\begin{equation}
\zeta^*_\cF(p,q)=\left(\operatorname{sup}_{\theta\in\Omega_\theta} ~\xi^r(p_{f_\theta},q_{f_\theta}) \right)^{\frac{1}{r}}
\label{eq:maxgspm}
\end{equation}
Verifying the metric properties for Eq. \eqref{eq:maxgspm} is trivial, given the properties of $\xi$ (see the supplementary material). Note that the recently proposed distances like Sliced Wasserstein (SW) distances and max-SW distances are a special case of GSPMs and Max-GSPMs. 

\section{GSPMs and MMDs} 

\label{sec:gspm_mmd}

The seminal work by \citet{gretton2007kernel,gretton2012kernel} on {\it maximum mean discrepancy} (MMD) provides a framework for efficient comparison of probability distributions. MMD is an {\it integral probability metric} \cite{sejdinovic2013equivalence}, and has become a popular choice of comparison between distributions in a wide variety of applications, e.g., generative modeling \cite{li2017mmd,tolstikhin2018wasserstein}, and gradient flows \cite{arbel2019maximum}. In practice, MMD is defined with  respect to a Reproducing Kernel Hilbert Space (RKHS), with a unique kernel. Like other kernel methods, the choice of kernel is often an application-dependent choice. In what follows, we show that an interesting family of GSPMs could be related to MMDs. Notably, we combine generalized slices together with a specific family of distances, which both have clear geometric interpretations, and obtain MMDs with well-defined kernels.

Consider Equation \eqref{eq:gspm} for the special case of $\xi(p_{f_\theta},q_{f_\theta})=\|Ap_{f_\theta}-Aq_{f_\theta}\|_2$ and $r=2$, where $A$ is a positive(-definite) linear operator. The positive assumption enforces $\xi$ to be a norm (i.e., the weighted Euclidean norm).  If $A$ is positive semi-definite, then $\xi$ would become a pseudo-metric, and as a consequence $\zeta_\cF$ also becomes a pseudo-metric. Given a linear operator, $A$, we can write: 
\begin{equation}
\zeta^2_\cF(p,q)=\int_{\Omega_\theta} \|Ap_{f_\theta}-Aq_{f_\theta}\|_2^2d\theta
\label{eq:lin}
\end{equation}

We focus on practical settings where we only observe samples $\{x_i\sim p\}_{i=1}^N$ and $\{y_j\sim q\}_{j=1}^M$ from these distributions. Substituting the empirical distribution in Equation \eqref{eq:gslice} give us the empirical slices as
$\hat{p}_{f_\theta}(t)=\frac{1}{N}\sum_{i=1}^N \delta(t-f_\theta(x_i))$ and $\hat{q}_{f_\theta}(t)=\frac{1}{M}\sum_{j=1}^M \delta(t-f_\theta(y_j))$. Using a common trick-of-trade in statistics, and without the loss of generality, we consider a smoothened version of the empirical slices via a {\it radial basis function} (RBF), $\phi_\sigma$, where $\sigma$ identifies the radius of the RBF ($\phi_{\sigma=0}(\cdot)=\delta(\cdot)$). Note that using $\phi_{\sigma}$ is equivalent to assuming smoothness priors on the slices.

By plugging in the (smoothened) empirical sliced distributions into \eqref{eq:lin}, we obtain:
{\small
\begin{eqnarray}
&\hspace{-2.5in}\zeta^2_\cF(\hat{p},\hat{q})= \nonumber\\
&\frac{1}{N^2}\sum_{ij} \underbrace{\int_{\Omega_\theta}\langle A\phi_\sigma(\cdot-f_\theta(x_i)),A\phi_\sigma(\cdot-f_\theta(x_j))\rangle d\theta}_{k(x_i,x_j)}+\nonumber\\
&\frac{1}{M^2}\sum_{ij} \underbrace{\int_{\Omega_\theta}\langle A\phi_\sigma(\cdot-f_\theta(y_i)),A\phi_\sigma(\cdot-f_\theta(y_j))\rangle d\theta}_{k(y_i,y_j)}-\nonumber\\
&\frac{2}{MN}\sum_{ij} \underbrace{\int_{\Omega_\theta}\langle A\phi_\sigma(\cdot-f_\theta(x_i)),A\phi_\sigma(\cdot-f_\theta(y_j))\rangle d\theta}_{k(x_i,y_j)} 
\label{eq:gspm_kernel}
\end{eqnarray}}
Equation \eqref{eq:gspm_kernel} is also the squared MMD with the particular kernel shown there-in.  Note that one can use the Monte-Carlo integral approximation to obtain an algorithmic way of calculating the kernel for any feasible $\cF$, $\phi_\sigma$, and $A$.  

We now argue that these family of kernels are positive definite (PD). Indeed,
\begin{equation}
    k_\theta(x_i,x_j):=\langle A\phi_\sigma(\cdot-f_\theta(x_i)),A\phi_\sigma(\cdot-f_\theta(x_j))\rangle
    \label{eq:k_theta}
\end{equation}
is a dot-product kernel, which is by definition PD, and summation/integration of PD kernels results in a PD kernel. Therefore,
\begin{equation}
    k(x_i,x_j):=\int_{\Omega_\theta} k_\theta(x_i,x_j) d\theta
    \label{eq:kernel}
\end{equation}
is a PD kernel. Below, we study some special interesting cases of the GSPMs based on $\xi(p_{f_\theta},q_{f_\theta})
=\|Ap_{f_\theta}-Aq_{f_\theta}\|_2$, and their equivalent MMD form based on kernels.

\subsection{First example: $A=id(\cdot)$}

When $A=id(\cdot)$, the GSPM is a generalized-sliced $\ell_2$ distance between the two distributions. This subsection shows that the work of \cite{knop2018cramerwold} follows this setting. In addition, we demonstrate that while such generalized sliced $\ell_2$ distance might not be as interesting for $\cF=\{f|f(\cdot)=\langle \cdot,\theta \rangle,~\forall\theta\in\S^{d-1}\}$, from a geometric point of view, it becomes appealing for more complex family of slices (e.g., homogeneous polynomials). 

Assuming the RBF is a Gaussian, $\phi_\sigma(t)=\N(0,\frac{\sigma}{2})(t)$ and using the inner product between two Gaussians, one can show that the dot-product kernel in Eq.~\eqref{eq:k_theta} boils down to:
\begin{equation}
    k_\theta(x_i,x_j)= \N(f_\theta(x_i)-f_\theta(x_j),
    \sigma)(0)
    \label{eq:k_theta_gaussian}
\end{equation}
The geometric interpretation of Equation \eqref{eq:k_theta_gaussian} is quite interesting. First note that $f_\theta:\cX\rightarrow \R$ therefore, the pre-image of a scalar in the range of $f_\theta$ is a \emph{hyper-surface} in $\cX$. This means that all points living on a hyper-surface would be projected to the same scalar in the range of $f_\theta$ (i.e., iso-hyper-surface). Therefore, while $x_i$ and $x_j$ could be far away from one another (in a Euclidean sense), as long as they live on the same or nearby iso-hyper-surfaces they will considered to be similar (with respect to $f_\theta$). Figure \ref{fig:hypersurfaces} demonstrates this effect and shows different $f_\theta$s, from family of linear functions parameterized by $\theta$ on a unit sphere (a), and family of polynomials of degree 5 (b), for which samples $x_i,x_j\in \R^2$ are considered near-by/far-away. 

\begin{figure}[t!]
    \centering
    \includegraphics[width=\linewidth]{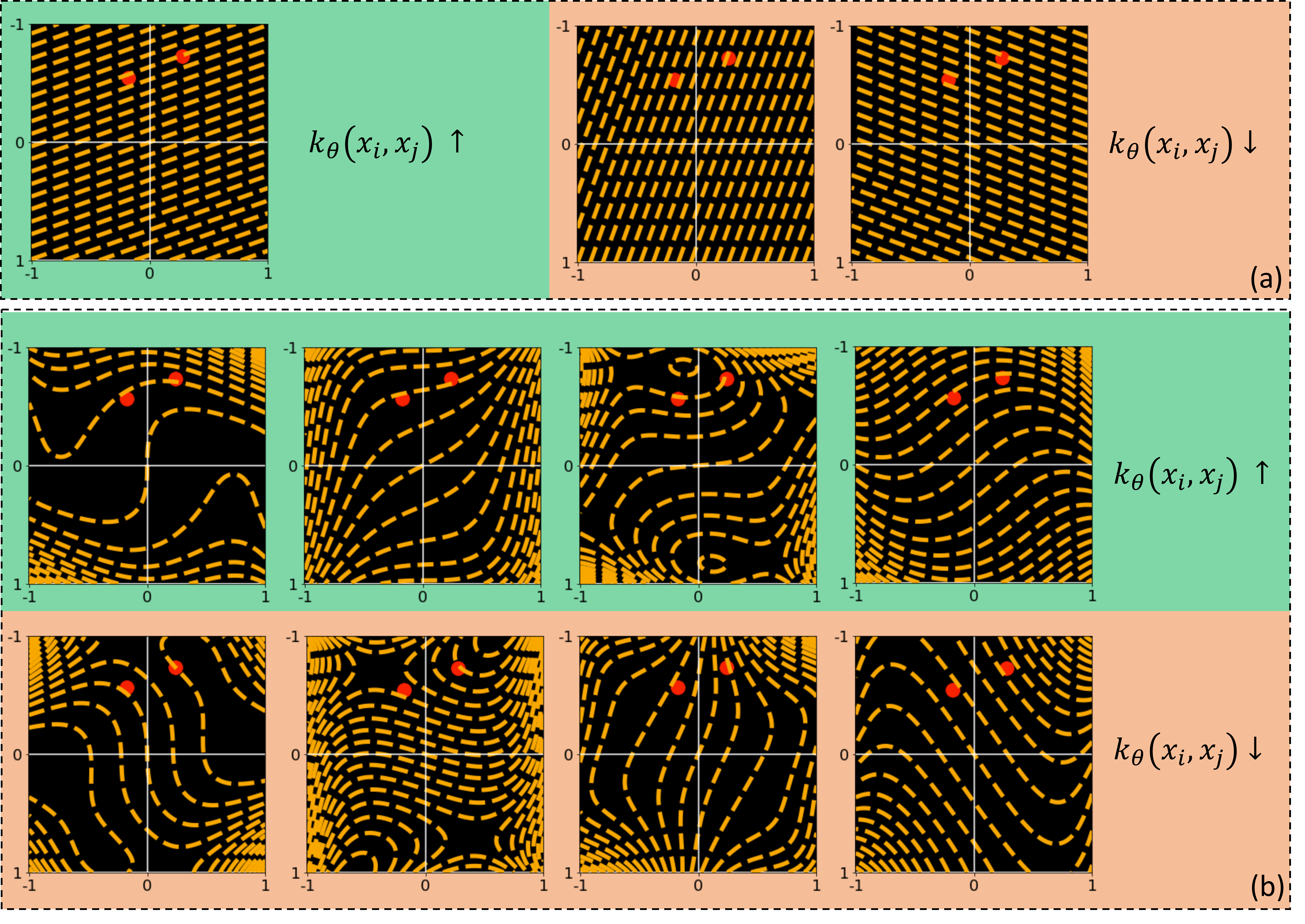}
    \caption{Visualization of two points $x_i$ and $x_j$ (red points), and the iso-hyper-surfaces (in 2D iso-curves) for sample $f_\theta\in\cF$, for $\cF=\{f|f(x)=\langle x,\theta\rangle,~\forall\theta\in\S^{d-1}\}$ (a) and $\cF=\{f|f(x)=\sum_{|\alpha|=5} \theta_\alpha x^\alpha,~\forall \theta\in\Omega_\theta\}$ (b). The green color indicates when $k_\theta(x_i,x_j)$ is high and the red color indicates when it is low.}
    \label{fig:hypersurfaces}
\end{figure}

As a special case, \citet{knop2018cramerwold} used linear slices (i.e., $f_\theta(x)=\langle x,\theta\rangle$) and showed that when $\phi_\sigma$ is the Gaussian function, then Equation \eqref{eq:kernel} has a closed form: 
$$
k(x_i,x_j)=\frac{1}{\sqrt{2\pi\sigma}}\psi_d(\frac{\|x_i-x_j\|^2_2}{2\sigma})
$$
where $\psi_d(\cdot)$ is a Kummer’s confluent hypergeometric function \cite{barnard1998gram} and can be approximated as:
\begin{equation}
    k(x_i,y_j)\approx\frac{1}{\sqrt{2\pi\sigma}}(1+\frac{\|x_i-y_j\|^2_2}{\sigma(d-\frac{3}{2})})^{-\frac{1}{2}}.
    \label{eq:cw}
\end{equation}

The above kernel also holds when $A$ is the Fourier transform, which is due to the fact that the Fourier transform is a unitary linear operator, i.e. satisfies $\langle A\phi_\sigma(\cdot-f_\theta(x_i)),A\phi_\sigma(\cdot-f_\theta(x_i))\rangle=\langle \phi_\sigma(\cdot-f_\theta(x_i)),\phi_\sigma(\cdot-f_\theta(x_i))\rangle$. However, note that the Fourier transform of a PDF is the characteristic function. Therefore, one would be considering L2-norm squared of the characteristics functions of the slices.

Recall that in these derivations, we started by fixing a slicing operation (linear slices), and used a specific distance, i.e. $\ell_2$ distance, and that we know the geometric meaning of both of these steps and their implications. Then, we ended up with a novel PD kernel that defines a MMD, which inherits these geometric properties. Here we emphasize that the distance used here (and in \cite{knop2018cramerwold}) is a Sliced-$\ell_2$. In the next section, we study the specific case of the Generalized-Sliced-Cram\'{e}r distance.  


\subsection{Second example: $A$ is the cumulative integral operator}

Now, we choose $A$ as the cumulative integral operator:
$$ Ap_{f_\theta}(t):=\int_{-\infty}^t p_{f_\theta}(\tau)d\tau.$$
Note that such $A$ is a positive definite operator. In this setting, the distance $\xi(p_{f_\theta},q_{f_\theta})=\|Ap_{f_\theta}-Aq_{f_\theta}\|_2$ is the 2-Cram\'{e}r distance \cite{Cramer1928composition} between the two one-dimensional probability distributions, $p_{f_\theta}$ and $q_{f_\theta}$, which is recently used in various publications \cite{bellemare2017cramer,kolouri2020sliced}. The Cram\'er distance shares some common characteristics to those of the Wasserstein distances. In fact, the 1-Cram\'er distance and the 1-Wasserstein distance are equivalent.  It is straightforward to show that $$k_\theta(x_i,x_j)=\langle A\phi_\sigma(\cdot-f_\theta(x_i)),A\phi_\sigma(\cdot-f_\theta(x_j))\rangle$$
is unbounded. Note that $A\phi_{\sigma}$ is the CDF of an RBF, and therefore its integral is unbounded. However, assuming that the integral domain is $[-T,T]$, we can find closed form solutions for $k_\theta(\cdot,\cdot)$. For instance, for $\phi_{\sigma=0}(\cdot)=\delta(\cdot)$ we have that $A\phi_0$ is a step function and, $$k_\theta(x_i,x_j)=T-\max(f_\theta(x_i),f_\theta(x_j)).$$

The boundedness assumption enforces us to use kernels $\phi_\sigma$ with a bounded range (hence, Gaussian kernels won't be allowed in this setting).  Our experiments indicate that smoothstep functions, often used in computer graphics, are well-suited candidates for $A\phi_\sigma$. The $n$'th order smoothstep function is defined as:
\begin{align*}
&A\phi_\sigma(x)=\\ &
 \left\{\begin{array}{lr}
     0 & x\leq-\sigma\\
    { \small \sum_{k=0}^n (-1)^k{n+k \choose k}{2n+1 \choose n-k}(\frac{x+\sigma}{2\sigma})^{n+k+1}}& |x|<\sigma\\
     1 & x\geq\sigma\\
\end{array}\right.   
\end{align*}
We include the derivations of $k_\theta(\cdot,\cdot)$, with the smoothstep functions, in the supplementary material. 


\subsection{Third example: $A$ is a generic integral transform} 
Integral transforms provide a broad family of linear operators, which could be used in Equation \eqref{eq:lin} to define novel distances/pseudo-distances (depending on the invertibility of the transform).  
The integral transform of a function, $\phi:\mathbb{R}\rightarrow\mathbb{R}$, is a generic linear transform defined as:
\begin{eqnarray}
A\phi(\cdot)=\int_{-\infty}^\infty \phi(x)\eta(x,\cdot)dx
\end{eqnarray}
where $\eta(x,z)$ is the integral kernel or the nucleus of the transform.  In this work, we suffice to mention this family of linear operators as an interesting class of operators for further studies. 

\section{GSPM Gradient Flows}

Gradient flows have become increasingly popular in implicit generative modeling \cite{csimcsekli2018sliced,arbel2019maximum,kolouri2019generalized}, where the aim is to minimize a functional in the Wasserstein space (i.e., the space of probability measures with bounded second-order moments, metrized by the Wasserstein-2 metric), given as follows:
\begin{align}
p^\star = \arg\min_p \zeta_{\mathcal{F}}^2(p,q). \label{eqn:optim}
\end{align}
In this section, we will exploit the connections that we developed between GSPMs and MMD (as detailed in Section~\ref{sec:gspm_mmd}) and develop a \emph{globally convergent} algorithm for solving problems of the form of \eqref{eqn:optim} by building up on the recent theoretical results given in \cite{arbel2019maximum}.

We now present the \emph{GSPM-flows}, that aim at generating a path of measures $(p_t)_{t\geq 0}$ which minimizes the squared GSPM between an initial measure $p_0$ and a target measure $q$ as $t$ goes to infinity. In particular, we will consider the gradient flow, informally expressed as follows:
\begin{align}
    \partial_t p_t = \nabla_{\mathcal{W}} \frac1{2}\zeta_{\mathcal{F}}^2(p_t,q), 
    \label{eqn:gf}
\end{align}
where $\nabla_{\mathcal{W}}$ denotes a notion of a gradient in the Wasserstein space \cite{ambrosio2008gradient}. Such gradient flows are of particular interest for generative modeling, since if the solution paths of the flow can be shown to converge to the global optimum $p^\star$,
then one can approximately simulate the gradient flow in order to solve the minimization problem and estimate $p^\star$.

Under appropriate conditions \cite{ambrosio2008gradient}, a path $(p_t)_{t \geq 0}$ is a solution of \eqref{eqn:gf} if and only if it solves a continuity equation of the form: 
\begin{align}
\partial_t p_t + \mathrm{div}(v p_t) = 0,    \label{eqn:cont}
\end{align}
$\mathrm{div}$ denotes the divergence operator and $v$ is a vector field, given as follows: \cite{arbel2019maximum}
\begin{align*}
    v(x,p) = - \nabla_x \Biggl( \int k(z,x) q(z) dz - \int k(z,x) p(z) dz \Biggr),
\end{align*}
where $k$ is defined in \eqref{eq:kernel}.

The partial differential equation representation \eqref{eqn:cont} has important practical implications, since such PDEs are often associated with a McKean-Vlasov (MV) process \cite{bogachev2015fokker}, which can be used for developing practical algorithms. In particular, associated to the continuity equation, we can define a MV process $(X_t)_{t\geq 0}$ as a solution to the following differential equation:
\begin{align}
    dX_t= v(X_t ,p_t)dt, \qquad   X_0 \sim p_0, \label{eqn:mv_proc}
\end{align}
where $X_t$ denotes the state of the process at time $t$. Here, $X_t$ evolves through the \emph{drift} function $v$, which requires the knowledge of $p_t$, i.e., the density function of $X_t$. The interest in this process is that the probability density functions of $(X_t)_t$ solve the continuity equation, hence, solving the optimization problem \eqref{eqn:optim} reduces to simulating \eqref{eqn:mv_proc}. 

Unfortunately, exact simulation of \eqref{eqn:mv_proc} is often intractable due to (i) the process is continuous-time, it needs to be discretized, (ii) the drift depends on the density $p_t$, which is not available in general. 
We will focus on the discretization of the process first, then we will develop a particle-based approach to alleviate the second problem. 

In order to discretize \eqref{eqn:mv_proc}, we consider the noisy Euler-Maruyama scheme, proposed in \cite{arbel2019maximum}, given as follows:
\begin{align}
    X_{n+1} = X_{n} + \eta v(X_n + \beta_n U_n, p_n), \label{eqn:em}
\end{align}
where $\eta>0$ is a step-size, $n=0,1,2,\dots$ denotes the iterations, $p_n$ denotes the density of $X_n$, $\beta_n >0$ denotes an inverse temperature variable, and $U_n$ is a standard Gaussian variable. If $\beta_n =0$ for all $n$, this scheme reduces to the standard Euler-Maruyama discretization, whereas a positive $\beta_n$ would drive the scheme to explore the space in a more efficient way. 


As one of our main contributions, we will now identify sufficient regularity conditions on the defining function $f_\theta$ and the smoothing function $\phi_\sigma$, which will be required for the convergence analysis of the gradient flow and its discretization \eqref{eqn:em}. 
\begin{condition}
\label{cond:c1}
    $A$ is a linear, bounded, positive semi-definite operator with the corresponding operator norm $\|A\|_{op}$.
\end{condition}

\begin{condition}
\label{cond:c2}
There exists a constant $G_f$, such that (for any $\theta\in \Omega_\theta$) $\|\nabla f_\theta(x)\|\leq G_f$ for all $x\in \cX$ and 
\begin{align}
\|\nabla f_\theta(x)-\nabla f_\theta(y)\|\leq G_f\|x-y\|
\end{align}
for all $x,y\in \cX$.
\end{condition}

\begin{condition}
\label{cond:c3}
There exists a constant $G_\phi$, such that the following inequalities hold:
 $|\phi_\sigma(\cdot)|\leq G_\phi$, $|\phi'_\sigma(\cdot)|\leq G_\phi$, $|\phi_\sigma(t)-\phi_\sigma(t')|\leq G_\phi|t-t'|$, and $|\phi'_\sigma(t)-\phi'_\sigma(t')|\leq G_\phi|t-t'|$.
\end{condition}

We now present our main result.
\begin{theorem}
\label{thm:main}
Let $p_0$ be a distribution with finite second-order moment. Then, under Conditions~\ref{cond:c1},\ref{cond:c2},\ref{cond:c3}, there exists a unique $(X_t)_{t\geq 0}$ solving \eqref{eqn:mv_proc} such that the density functions of $(X_t)_{t \geq 0}$ constitute the unique solution of \eqref{eqn:cont}.

Furthermore, let $(X_n)_{n\in\mathbb{N}_+}$ be the iterates obtained by \eqref{eqn:em}. If $\sum_{i=1}^n \beta_i^2 \to \infty$ as $n \to \infty$, then the following bound holds:
\begin{align}
\zeta(p_n,q) \leq \zeta(p_0,q)e^{-2 \lambda^{2} \eta(1-3 \eta L) \sum_{i=0}^{n} \beta_{i}^{2}},  
\end{align}
where $p_n$ denotes the density of $X_n$ and
\begin{align*}
    L&=(G^2_f+G_f)G^2_\phi\|A\|^2_{op}\\
    \lambda&=\Bigl(2d\|A\|^2_{op}G^2_\phi G^2_f(1+G^2_f)\Bigr)^{1/2}.
\end{align*}
\end{theorem}
%
The proof is given in the supplement. This result shows that, with sufficiently regular $f_\theta$ and $\phi_\sigma$, the noisy Euler scheme \eqref{eqn:em} can achieve the global optimum, where the convergence rate depends on the structure of $f_\theta$ and $\phi_\sigma$.

Even though Theorem~\ref{thm:main} hints the potential of the proposed gradient flow, the discretization scheme \eqref{eqn:em} is unfortunately still intractable due to the dependency of $v$ on $p_n$. In order to obtain a practical algorithm, we finally consider a \emph{particle system} that serves as an approximation to the original system \eqref{eqn:em}, and given as follows:
\begin{align}
    X^{i}_{n+1} = X^{i}_{n} + \eta v(X^{i}_n + \beta_n U^i_n, \hat{p}_n),
    \label{eq:update}
\end{align}
where $i=1,\dots,N$ denotes the particle index and $\hat{p}_n = \frac1{N} \sum_{j=1}^N \delta_{X_n^i}$ denotes the \emph{empirical distribution} of $\{X^i_n\}_{i=1}^N$. Here, the idea is to approximate $p_n$ by $\hat{p}_n$ by evolving $N$ different particles at the same time. Similar schemes have proved successful in generative modeling \cite{csimcsekli2018sliced} and Bayesian machine learning \cite{liu2016stein}. Moreover, one can further show that the particle system converges to the true system \eqref{eqn:em} with a rate of $N^{-1/2}$ \cite{durmus2018elementary,arbel2019maximum}.  

\section{Numerical Experiments}

In this section, we present our experimental results that illustrate our framework. 

\begin{figure}[t!]
    \centering
    \includegraphics[width=\columnwidth]{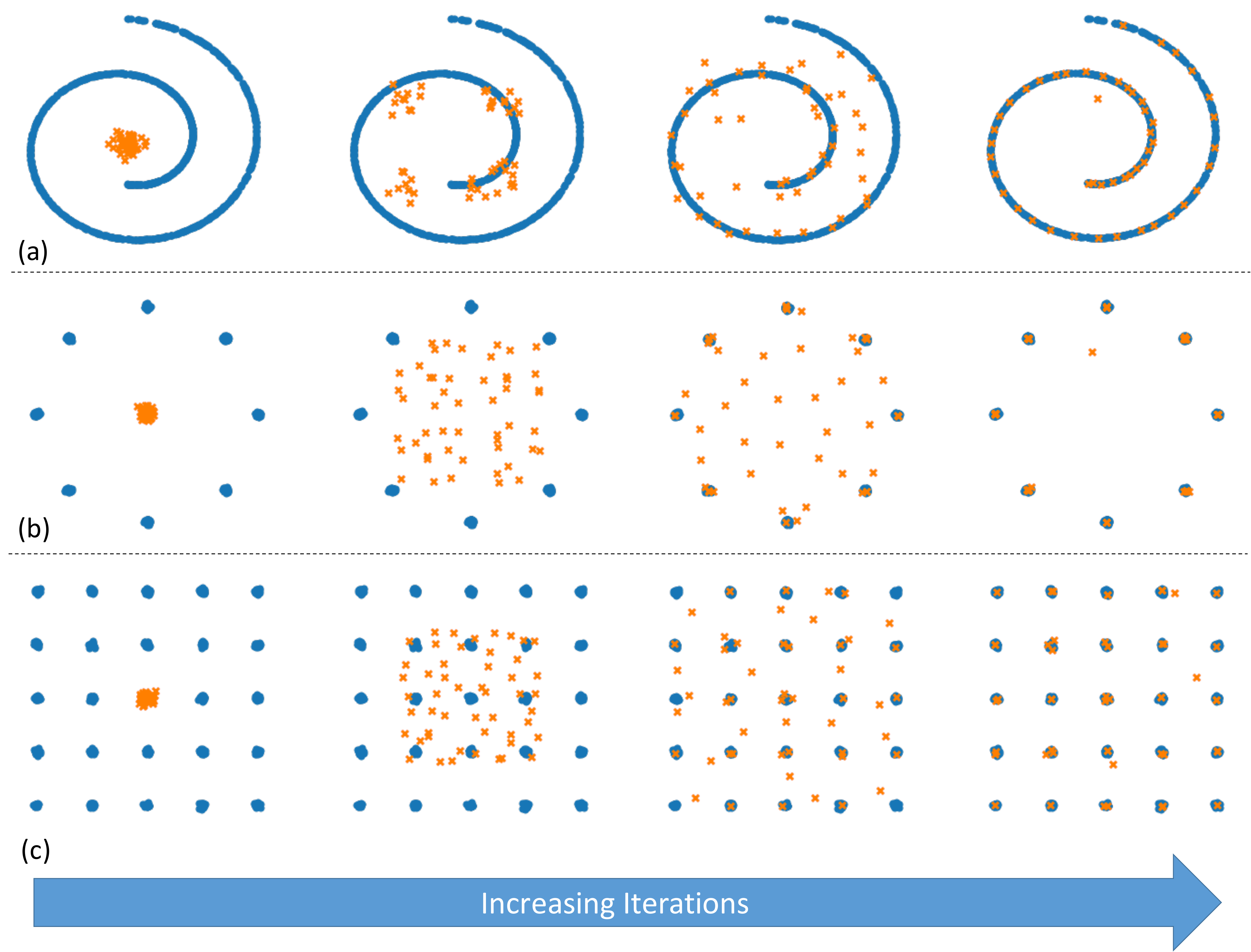}
    \caption{Gradient flow on the Swiss Roll distribution (a), the 8-Gaussians distribution (b), and the 25-Gaussians distribution (c), using the proposed GSPM-MMD kernels with $A=id(\cdot)$. The source distribution consists of $N=50$ particles.}
    \label{fig:flow}
\end{figure}
\begin{figure*}[t!]
    \centering
    \includegraphics[width=\linewidth]{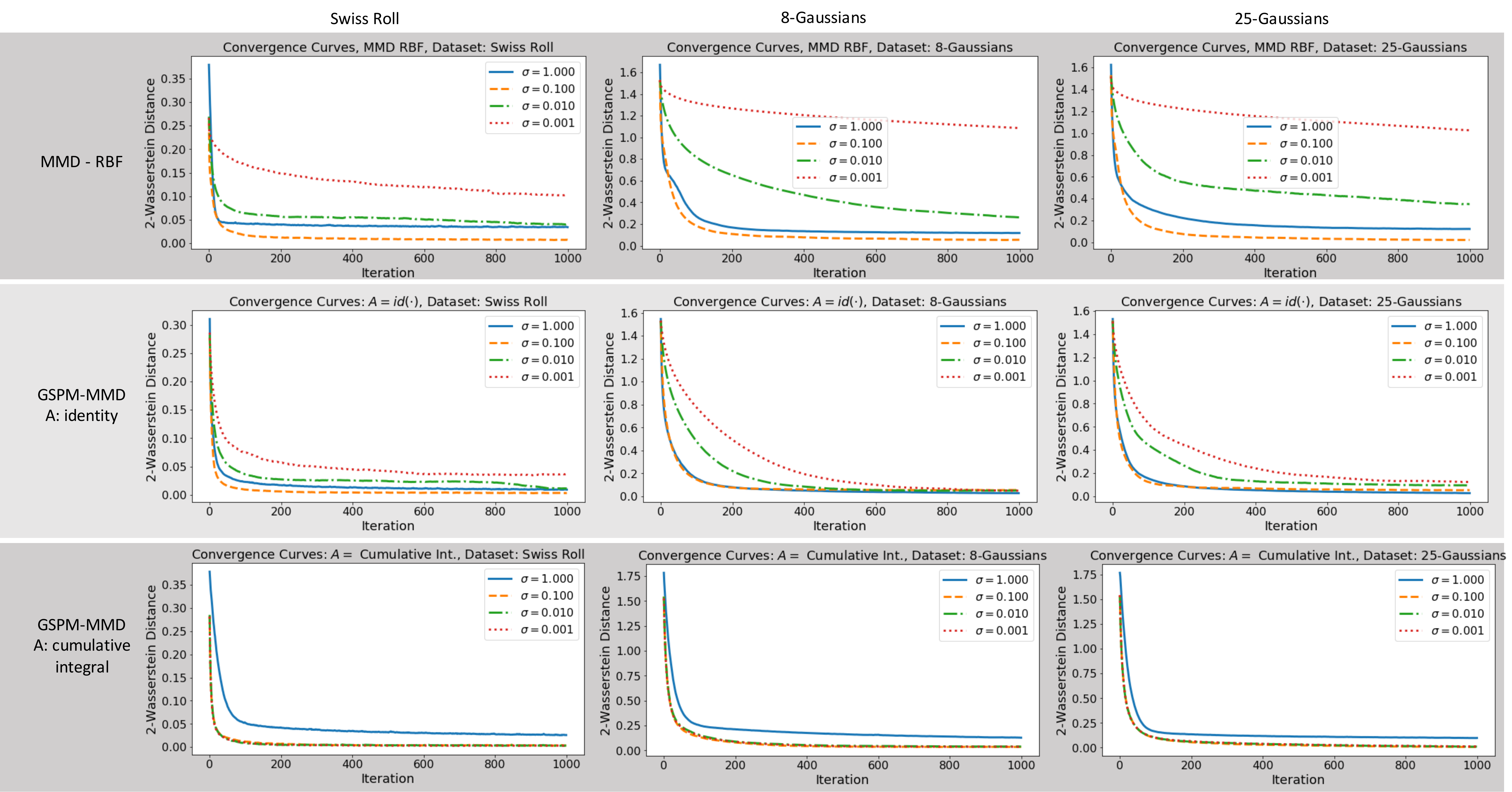}
    \vspace{-20pt}
    \caption{The convergence curves of MMD flows for MMD-RBF, GSPM-MMD with identity operator, and GSPM-MMD with the cumulative integral operator on the three synthetic datasets shown in Figure \ref{fig:flow}.}
    \vspace{-5pt}
    \label{fig:synthetic}
\end{figure*}

\subsection{Gradient flow -- Synthetic}

We first perform a numerical experiment with synthetic datasets to demonstrate the performance of the proposed GSPM-MMD kernels. To simplify the presentation, in our first experiment, we assumed the noise $\beta_n=0$ for all $n$ (i.e., the standard Euler-Maruyama discretization). We consider three two-dimensional target distributions, namely the Swiss Roll, the 8-Gaussians, and the 25-Gaussians distributions. The source distribution is initialized with $N$ samples from a Gaussian distribution.  Figure \ref{fig:flow} shows the datasets and the flow (calculated using GSPM-MMD). We calculate the gradient flow updates (See Equation \eqref{eq:update}) to match the source and the target distributions.

For our method, we used the GSPM-MMD kernel with $A=id(\cdot)$ and when $A$ is the cumulative integral operator (i.e., the 2-Cram\'{e}r distance). For simplicity, we used linear slices $f_\theta(x)=\theta\cdot x$. Also, as a standard baseline for comparison, we apply the Gaussian kernel and minimize the MMD flow. In each iteration of the gradient flow, we measure the 2-Wasserstein distance between the updated source and the target distribution. For each method we vary $\sigma\in [0.001,0.01,0.1,1]$, and repeat the experiments $10$ times. Figure \ref{fig:synthetic} compares the algorithms on the three datasets and for various $\sigma$. For the cumulative integral operator we used $L=10$ slices.

{\bf Effect of noise:} The addition of noise lessens the effect of a poor choice of $\sigma$ by allowing the particles to explore the space in a more efficient manner. 
To demonstrate the effect of the addition of noise to the updates (See Equation \eqref{eq:update}) we repeated the experiment in Figure \ref{fig:flow} for the Swiss Roll dataset, but with a poor choice of $\sigma$. From Figure \ref{fig:flow}, one can see that $\sigma=0.001$ is too small for calculating an effective flow. Hence, we chose $\sigma=0.001$ and solved a noisy gradient flow problem with the MMD-RBF kernel (baseline) and the GSPM-MMD with $A=id(\cdot)$ kernel. We selected the initial $\beta\in[1,0.1,0.01,0]$ and decayed the noise in each gradient iteration with a $1/k$ rate ($k$ being the iteration). The log 2-Wasserstein between the source and target distributions is depicted in Figure \ref{fig:noise}. As expected, addition of noise improves the overall performance of gradient flows.

{\bf Linear or non-linear slices:} So far, in our experiments, we have only used  linear slices, i.e., $f_\theta(x)=\langle x,\theta\rangle$.   Here, we compare GSPM-MMD flows solely based on the choice of linear and non-linear slices. For the non-linear slices, in this experiment, we use the homogeneous polynomials of degree 5 (see Figure \ref{fig:hypersurfaces}). To ensure a fair comparison, we chose the number of random slices for both GSPM-MMD kernels to $L=1$. Figure \ref{fig:slices} shows the comparison between linear and polynomial slices. 

\begin{figure}[t]
    \centering
    \includegraphics[width=\columnwidth]{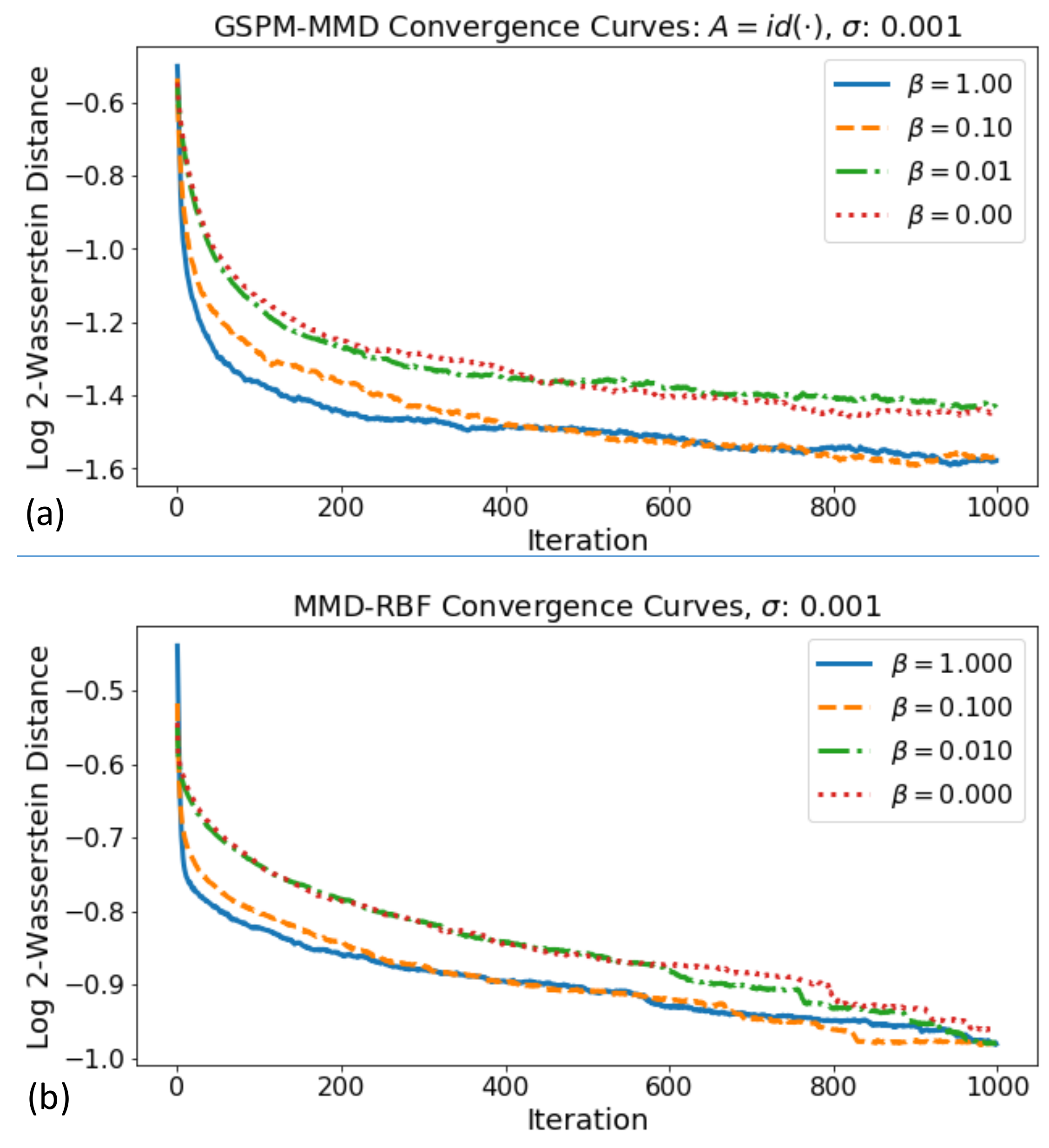}
    \vspace{-20pt}
    \caption{The effect of the addition of Gaussian noise (See Equation \eqref{eq:update}) in calculating the MMD flows using the GSPM-MMD kernel with $A=id(\cdot)$ (a), and the RBF kernel (b). The results are averaged over 10 runs and are calculated on the Swiss Roll dataset.}
    \label{fig:noise}
\end{figure}
\begin{figure}
    \centering
    \includegraphics[width=\columnwidth]{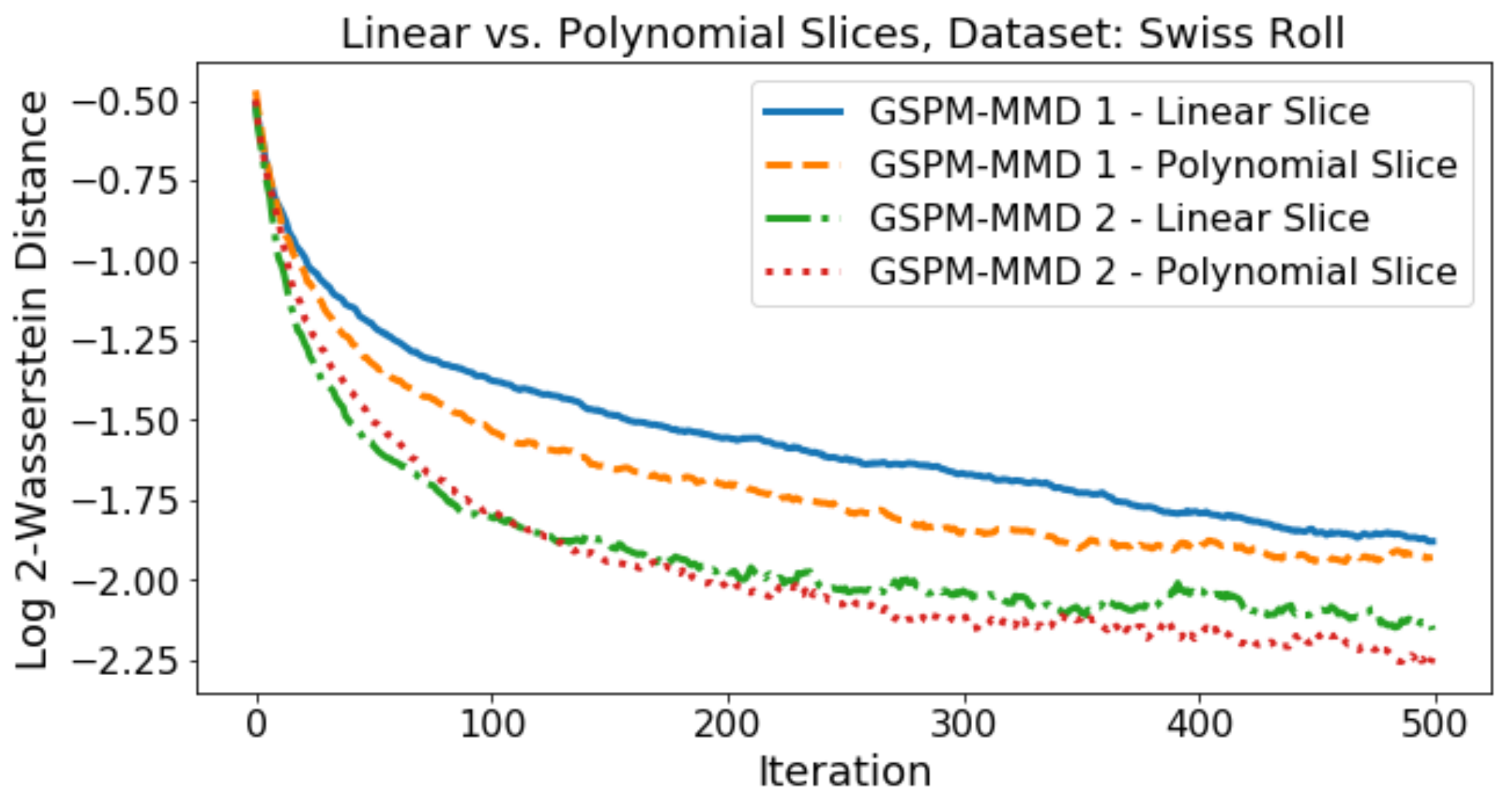}
    \vspace{-20pt}
    \caption{Gradient flows using GSPM-MMDs with linear and Polynomial slices. The experiment was calculated based on one random slice (linear or nonlinear) per iteration.}
    \label{fig:slices}
\end{figure}

\subsection{Gradient flow -- MNIST}

To show the effectiveness of the proposed distances in higher dimensions, we designed the following experiment. We first learn a simple convolutional auto-encoder, with an added classifier on its bottleneck to ensure a discriminative space embedding, to embed the MNIST dataset into a $(d=16)$-dimensional space. Then we solve the gradient flow problem in the embedded space with $N=100$ particles initialized from a Gaussian distribution.

Similar to the previous experiments, we use MMD-RBF, GSPM-MMD with $A=id(\cdot)$ (denoted as GSPM-MMD 1), and GSPM-MMD with $A$ being the cumulative integral (denoted as GSPM-MMD 2) and calculate the flow between the source and target distributions. We measure the 2-Wasserstein distance between the distributions at each iteration. The experiments were repeated $10$ times and the average performance for each method is reported in Figure \ref{fig:MNIST} (top row). After the convergence, we sort the particles according to the output of the classifier and feed them to the decoder network to visualize the corresponding digits for each method (See the bottom row in Figure \ref{fig:MNIST}). We note that same $\sigma$ was used for all three methods, and linear slicing was used in this experiment. We conclude that the GSPM-MMD with the cumulative integral operator, which corresponds to the sliced-Cram\'{e}r distance, seems to achieve a superior performance in comparison with the other two kernels.

\section{Conclusion}

We introduced a new family of distances, denoted as Generalized Sliced Probability Metrics (GSPMs), which calculate the expected distances between slices (i.e., one-dimensional marginals) of two input distributions. We then showed that a subset of the proposed distances is equivalent to the squared maximum mean discrepancy (MMD) with new kernels introduced in this work, denoted as GSPM-MMD kernels. Furthermore, we applied the GSPM-MMD kernels in the domain of gradient flows for implicit generative modeling, which has recently attracted ample attention from the research community. More importantly, we identified sufficient regularity conditions on the building elements of our proposed distance (and consequently the proposed kernels) for guaranteeing global convergence of the gradient flow. Finally, we provide extensive ablation experiments to test our proposed distance on synthetic and real datasets. 
 \begin{figure}[t!]
    \centering
    \includegraphics[width=\columnwidth]{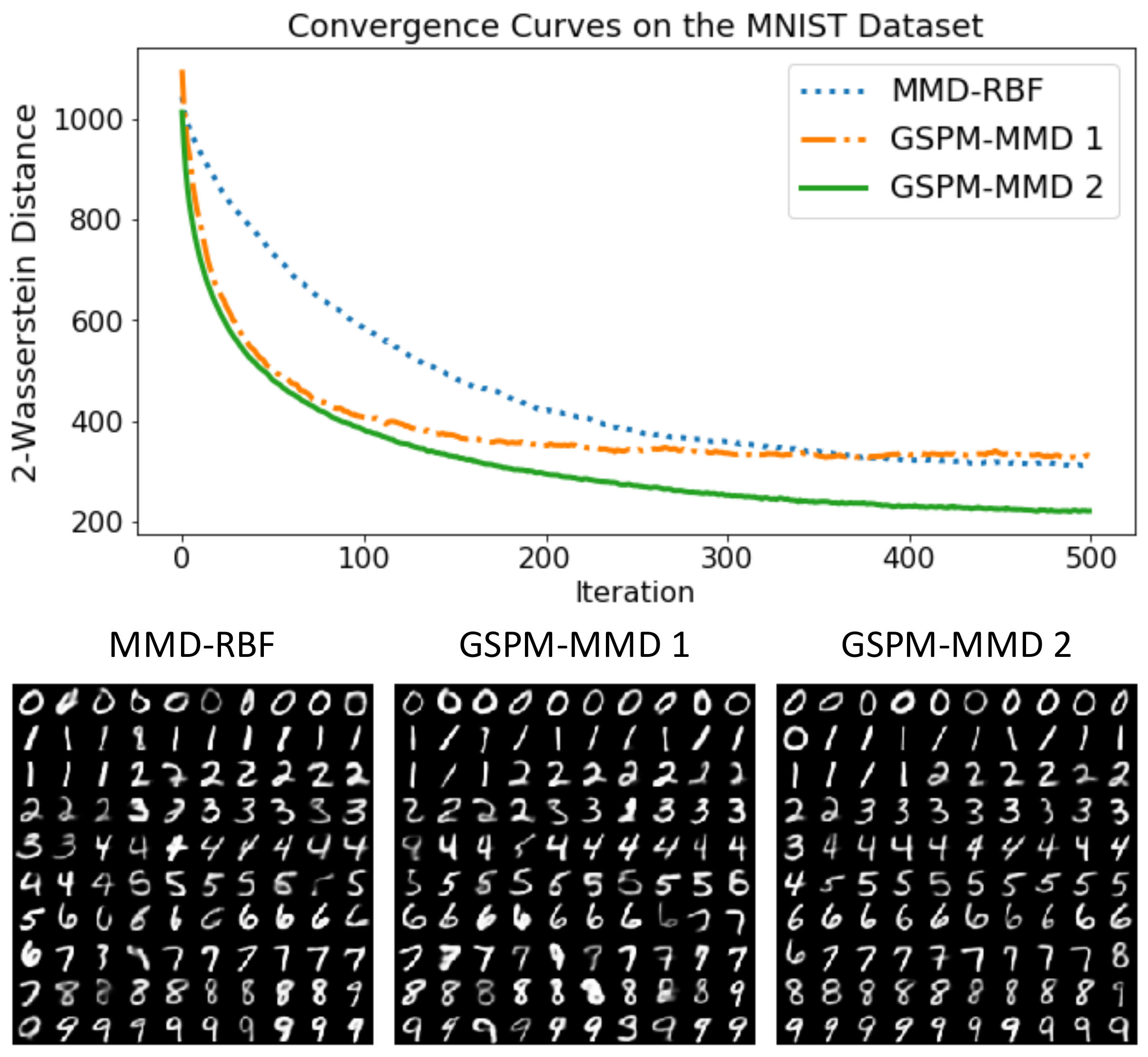}
    \caption{Comparison of the proposed kernels on the MNIST dataset. The top row shows the 2-Wasserstein distance between the source and target distributions, while the bottom row visualizes the actual target particles, $N=100$.}
    \label{fig:MNIST}
\end{figure}

\bibliographystyle{icml2020}
\bibliography{GSD_PD.bib}

\clearpage
\section*{Supplementary Material}
We invoke the following lemma from \cite{steinwart2008support} to prove our result. 
\begin{lemma}\label{L2}(Lemma 4.34 in \cite{steinwart2008support})
Let $\cX\subset \mathbb{R}^d$ be an
open subset, $k$ be a kernel on $\cX$, $\cH_k$ be a feature space of $k$, and $\Phi: \cX \to \cH_k$ be a feature map of $k$. Let $i\in \{1,\ldots,d\}$ be an index such that the mixed partial derivative $\partial_i\partial_{i+d}k$ of $k$ with respect to the coordinates $i$ and $i+d$ exists and is continuous. Then the partial derivative $\partial_i\Phi$ with respect to the $i$-th coordinate exists, is continuous, and for all $x,x'\in \cX$ we have
\begin{align*}
    \langle\partial_i \Phi(x),\partial_i \Phi(y)\rangle_{\mathcal{H}_k}=\partial_i\partial_{i+d}k(x,y)=\partial_{i+d}\partial_i k(x,y).
\end{align*} 
\end{lemma}

\begin{lemma}\label{L3}
Let us define $\psi_x(\cdot)\triangleq A\phi_\sigma(\cdot-f_\theta(x))$ such that $k_\theta$ in \eqref{eq:k_theta} can be represented as follows
$$
k_\theta(x,y)=\langle\psi_x(\cdot),\psi_y(\cdot)\rangle
$$
for any $\theta \in \Omega_\theta$. Then, under Conditions~\ref{cond:c1}-\ref{cond:c3}, we have that
\begin{itemize}
    \item $\|\psi_x(\cdot)\|\leq \|A\|_{op}G_\phi$.
    \item $\|\psi'_x(\cdot)\|\leq \|A\|_{op}G_\phi$.
    \item $\|\psi_x(\cdot)-\psi_y(\cdot)\|\leq \|A\|_{op}G_\phi G_f\|x-y\|$.
    \item $\|\psi'_x(\cdot)-\psi'_y(\cdot)\|\leq \|A\|_{op}G_\phi G_f\|x-y\|$.
\end{itemize}
\end{lemma}
\begin{proof}
The proof of statements above follows immediately from Conditions~\ref{cond:c1}-\ref{cond:c3}.
\end{proof}

\subsection*{Proof of Theorem \ref{thm:main}}

We prove that the kernel $k$ in \eqref{eq:kernel} has $L$-Lipschitz gradients:
\begin{align}
    \|\nabla k(x,x') - \nabla k(y,y')\| \leq L( \|x-y\| + \|x'-y'\|),\label{eqn:cond_kern1}
\end{align}
and satisfies the following inequality:
\begin{align}
    \sum_{i=1}^d \|\partial_i k(x,\cdot) -\partial_i k(y,\cdot) \|^2_{\mathcal{H} } \leq \lambda^2 \|  x-y\|^2. \label{eqn:cond_kern2}
\end{align}
Then the rest of the proof follows from  \cite{arbel2019maximum}, Proposition 1 (existence and uniqueness) and Proposition 8 (convergence of the Euler scheme).

Recalling the definition of $k_\theta$ from Lemma \ref{L3}, we have that 
\begin{align*}
    \nabla_xk_\theta(x,x')&-\nabla_xk_\theta(x,y')=\nabla_x\langle \psi_x(\cdot),\psi_{x'}(\cdot)-\psi_{y'}(\cdot)\rangle\\
    &=\nabla_xf_\theta(x)\langle \psi'_x(\cdot),\psi_{x'}(\cdot)-\psi_{y'}(\cdot)\rangle.
\end{align*}
Applying Lemma \ref{L3}, we can simplify above to get
\begin{align*}
  &\|\nabla_xk_\theta(x,x')-\nabla_xk_\theta(x,y')\|\\
  &~~~~~~~~~~~~~~\leq G_f\|\psi'_x(\cdot)\|\|\psi_{x'}(\cdot)-\psi_{y'}(\cdot)\|\\
  &~~~~~~~~~~~~~~\leq G^2_fG^2_\phi\|A\|^2_{op}\|x'-y'\|. \numberthis \label{eq:12}
\end{align*}
On the other hand, 
\begin{align*}
    &\nabla_xk_\theta(x,y')-\nabla_yk_\theta(y,y')\\
    &=\nabla_xf_\theta(x)\langle \psi'_x(\cdot),\psi_{y'}(\cdot)\rangle-\nabla_yf_\theta(y)\langle \psi'_y(\cdot),\psi_{y'}(\cdot)\rangle\\
    &=\nabla_xf_\theta(x)\langle \psi'_x(\cdot),\psi_{y'}(\cdot)\rangle-\nabla_yf_\theta(y)\langle \psi'_x(\cdot),\psi_{y'}(\cdot)\rangle\\
    &+\nabla_yf_\theta(y)\langle \psi'_x(\cdot),\psi_{y'}(\cdot)\rangle-\nabla_yf_\theta(y)\langle \psi'_y(\cdot),\psi_{y'}(\cdot)\rangle\rangle. 
\end{align*}
Due to Lipschitz continuity of $\nabla f_\theta$ as well as Lemma \ref{L3}, the above entails that
\begin{align*}
    &\|\nabla_xk_\theta(x,y')-\nabla_yk_\theta(y,y')\|\\
    &~~~~~~~~\leq \|\langle \psi'_x(\cdot),\psi_{y'}(\cdot)\rangle\|\|\nabla_xf_\theta(x)-\nabla_yf_\theta(y)\|\\
    &~~~~~~~~+\|\nabla_yf_\theta(y)\|\|\psi_{y'}(\cdot)\|\|\psi'_x(\cdot)-\psi'_y(\cdot)\|\\
    &~~~~~~~~\leq G_fG^2_\phi\|A\|^2_{op}\|x-y\|+G^2_fG^2_\phi\|A\|^2_{op}\|x-y\|
\end{align*}
Combining above with \eqref{eq:12}, we have by triangle inequality that
\begin{align*}
&\|\nabla_xk_\theta(x,x')-\nabla_yk_\theta(y,y')\|\\
&~~~~~~~~~~\leq(G^2_f+G_f)G^2_\phi\|A\|^2_{op}(\|x-y\|+\|x'-y'\|).
\end{align*}
Integrating above over $\Omega_\theta$ and interchanging the integral with the norm on the left-hand-side proves Condition \ref{eqn:cond_kern1} with $L=(G^2_f+G_f)G^2_\phi\|A\|^2_{op}$.

To prove that \eqref{eqn:cond_kern2} holds, we can use Lemma \ref{L2} to observe that
\begin{align*}
    \vphantom{\Big\|a\Big\|^2_a}&\Big\|\partial_i k(x,\cdot) -\partial_i k(y,\cdot) \Big\|^2_{\mathcal{H}_k}\\
    \vphantom{\Big\|a\Big\|^2_a}&~~~~~~=\partial_i\partial_{i+d} k(x,x)+\partial_i\partial_{i+d}k(y,y)-2\partial_i\partial_{i+d} k(x,y)\\
    \vphantom{\Big\|a\Big\|^2_a}&~~~~~~=\int_{\Omega_\theta} \partial_i\partial_{i+d}k_\theta(x,x)d\theta+\int_{\Omega_\theta} \partial_i\partial_{i+d}k_\theta(y,y)d\theta\\
    \vphantom{\Big\|a\Big\|^2_a}&~~~~~~-2\int_{\Omega_\theta} \partial_i\partial_{i+d}k_\theta(x,y)d\theta\\
    \vphantom{\Big\|a\Big\|^2_a}&~~~~~~=\int_{\Omega_\theta}\Big\|\partial_i k_\theta(x,\cdot) -\partial_i k_\theta(y,\cdot) \Big\|^2_{\mathcal{H}_{{k}_{\theta}}}d\theta.\numberthis \label{eq:11} 
\end{align*} 
We now have    
\begin{align*}
    \vphantom{\Big\|a\Big\|^2_a}&\Big\|\partial_i k_\theta(x,\cdot) -\partial_i k_\theta(y,\cdot) \Big\|^2_{\mathcal{H}_{{k}_{\theta}}}\\
    \vphantom{\Big\|a\Big\|^2_a}&~~~~~~=\partial_i\partial_{i+d}k_\theta(x,x)+ \partial_i\partial_{i+d}k_\theta(y,y)-2\partial_i\partial_{i+d}k_\theta(x,y)\\
    \vphantom{\Big\|a\Big\|^2_a}&~~~~~~=\langle\partial_if_\theta(x)\psi'_x(\cdot),\partial_if_\theta(x)\psi'_x(\cdot)\rangle\\
    \vphantom{\Big\|a\Big\|^2_a}&~~~~~~+\langle\partial_if_\theta(y)\psi'_y(\cdot),\partial_if_\theta(y)\psi'_y(\cdot)\rangle\\
    \vphantom{\Big\|a\Big\|^2_a}&~~~~~~-2\langle\partial_if_\theta(x)\psi'_x(\cdot),\partial_if_\theta(y)\psi'_y(\cdot)\rangle\\
    \vphantom{\Big\|a\Big\|^2_a}&~~~~~~=\Big\|\partial_if_\theta(x)\psi'_x(\cdot)-\partial_if_\theta(y)\psi'_y(\cdot)\Big\|^2\\
    \vphantom{\Big\|a\Big\|^2_a}&~~~~~~\leq2\Big\|\partial_if_\theta(x)\psi'_x(\cdot)-\partial_if_\theta(y)\psi'_x(\cdot)\Big\|^2\\
  \vphantom{\Big\|a\Big\|^2_a}    \vphantom{\Big\|a\Big\|^2_a}&~~~~~~+2\Big\|\partial_if_\theta(y)\psi'_x(\cdot)-\partial_if_\theta(y)\psi'_y(\cdot)\Big\|^2\\
    \vphantom{\Big\|a\Big\|^2_a}&~~~~~~\leq2\|A\|^2_{op}G^2_\phi\Big|\partial_if_\theta(x)-\partial_if_\theta(y)\Big|^2\\
    \vphantom{\Big\|a\Big\|^2_a}&~~~~~~+2G^2_f\big\|\psi'_x(\cdot)-\psi'_y(\cdot)\big\|^2\\
  \vphantom{\Big\|a\Big\|^2_a} &~~~~~~\leq2\|A\|^2_{op}G^2_\phi G^2_f(1+G^2_f)\big\|x-y\big\|^2.
\end{align*}
Integrating above uniformly over $\Omega_\theta$, substituting it in \eqref{eq:11}, and summing over $i$ proves that \eqref{eqn:cond_kern2} holds with  $\lambda^2=2d\|A\|^2_{op}G^2_\phi G^2_f(1+G^2_f)$.

\subsection{$k_\theta(x_i,x_j)$ for the smooth step function.}

Here we derive the analytical form of $k_\theta$ for the zero'th order smooth step function, which is essentially the clamping function. Note that similar analytical formulas could be derived for higher order smooth step functions. For simplicity we let $g_i:=A\phi_\sigma(\cdot-f(x_i))$, and without loss of generality, we assume that $f(x_i)\leq f(x_j)$. For the zero'th order smooth step function, assuming bounded range $[-T,T]$, we have:
\begin{eqnarray}
A\phi_\sigma(x)=\left\{\begin{array}{lr}
0 & -T<x\leq -\sigma\\
\frac{x+\sigma}{2\sigma} & |x|<\sigma\\
1 &  \sigma\leq x <T
\end{array}\right.
\end{eqnarray}
Then, we have $k_\theta(x_i,x_j)=\langle g_i, g_j \rangle$. We also assume that $max_x f(x)\leq T-\sigma$.

{\bf If $f(x_i)+\sigma \leq f(x_j)-\sigma$:} In this case we have 
\begin{eqnarray*}
k_\theta(x_i,x_j)&=&\int_{-T}^T g_i(t)g_j(t)dt \\
&=& \int_{f(x_j)-\sigma}^T g_j(t)dt \\
&=& \int_{f(x_j)-\sigma}^{f(x_j)+\sigma} \frac{t-f(x_j)+\sigma}{2\sigma} dt+\int_{f(x_j)+\sigma}^T 1 dt \\
&=& \sigma+T-f(x_j)-\sigma = T-f(x_j)
\end{eqnarray*}

{\bf If $f(x_j)-\sigma < f(x_i)+\sigma$:} For simplicity let $f_i:=f(x_i)$, then we have:
\begin{eqnarray*}
k_\theta(x_i,x_j)&=&\int_{-T}^T g_i(t)g_j(t)dt \\
&=& \frac{1}{4\sigma^2}\int_{f_j-\sigma}^{f_i+\sigma} (t+\sigma-f_i)(t+\sigma-f_j)dt+\\
&& \frac{1}{2\sigma}\int_{f_i+\sigma}^{f_j+\sigma} (t+\sigma-f_i)dt +\int_{f_j+\sigma}^T 1 dt 
\end{eqnarray*}
where,
\begin{eqnarray*}
\frac{1}{4\sigma^2}\int_{f_j-\sigma}^{f_i+\sigma} (t+\sigma-f_i)(t+\sigma-f_j)dt=\\ \frac{(f_j-f_i)^3-12\sigma^2(f_j-f_i)+16\sigma^3}{24\sigma^2},
\end{eqnarray*}
and,
\begin{eqnarray*}
\frac{1}{2\sigma}\int_{f_i+\sigma}^{f_j+\sigma} (t+\sigma-f_i)dt = \\
\frac{(f_j-f_i)(f_i-f_j+4\sigma)}{4\sigma}
\end{eqnarray*}
and,
\begin{eqnarray*}
\int_{f_j+\sigma}^T 1 dt =T-f_j-\sigma 
\end{eqnarray*}
remember that $f_i\leq f_j$. Then, we can write:
$$
k_\theta(x_i,x_j)=T-f(x_j)+\frac{(f(x_j)-f(x_i)-2\sigma)^3}{24\sigma^2}
$$
Putting it all together:
\begin{eqnarray*}
&&k_\theta(x_i,x_j)=\\ &&\left\{\begin{array}{lr}
    T-f(x_j) &  f(x_i)\leq f(x_j)-2\sigma \\
    T-f(x_j)+\frac{(f(x_j)-f(x_i)-2\sigma)^3}{24\sigma^2} & O.W.  
\end{array}    \right.
\end{eqnarray*}

\end{document}